\documentclass{article}

\usepackage{microtype}
\usepackage{graphicx}
\usepackage{booktabs} 

\usepackage{hyperref}

\usepackage{amsfonts}       
\usepackage{amsmath}        
\usepackage{amsthm}
\usepackage{subcaption}
\usepackage{hhline}
\usepackage{placeins}

\newtheorem{theorem}{Theorem}[section]

\newtheorem{lemma}[theorem]{Lemma}
\theoremstyle{definition}
\newtheorem{definition}[theorem]{Definition}

\newcommand\inner[1]{\langle #1 \rangle}
\newcommand\quant[1]{\tilde{#1}}    

\newcommand{\R}{{\mathbb{R}}}  
\DeclareMathOperator*{\argmax}{arg\,max}
\DeclareMathOperator*{\argmin}{arg\,min}

\RequirePackage{fancyhdr}
\RequirePackage{color}
\RequirePackage{algorithm}
\RequirePackage{algorithmic}
\RequirePackage{natbib}
\RequirePackage{eso-pic} 
\RequirePackage{forloop}
\usepackage[letterpaper, portrait, margin=1in]{geometry}
\usepackage{authblk}
\usepackage{enumitem}
\usepackage[symbol]{footmisc}

\title{Accelerating Large-Scale Inference with Anisotropic Vector Quantization}
\date{\vspace{-5ex}}

\author{\small Ruiqi Guo*}
\author{\small Philip Sun*}
\author{\small Erik Lindgren*}
\author{\small Quan Geng}
\author{\small David Simcha}
\author{\small Felix Chern}
\author{\small Sanjiv Kumar}
\affil{Google Research}
\affil{\texttt{\{guorq, sunphil, erikml, qgeng, dsimcha, fchern, sanjivk\}@google.com}}

\usepackage{titlesec}

\titleformat*{\section}{\large\bfseries\raggedright}
\titleformat*{\subsection}{\normalsize\bfseries\raggedright}
\titleformat*{\subsubsection}{\normalsize\scshape\raggedright}

\setlist{leftmargin=1.8em}

\begin{document}

\setlength{\columnsep}{16pt}
\twocolumn[
\maketitle

\vskip 0.3in
]

\renewcommand{\thefootnote}{\fnsymbol{footnote}}
\footnotetext[1]{Equal contributions.}

\begin{abstract}
Quantization based techniques are the current state-of-the-art for scaling maximum inner product search to massive databases. Traditional approaches to quantization aim to minimize the reconstruction error of the database points. Based on the observation that for a given query, the database points that have the largest inner products are more relevant, we develop a family of anisotropic quantization loss functions. Under natural statistical assumptions, we show that quantization with these loss functions leads to a new variant of vector quantization that more greatly penalizes the parallel component of a datapoint's residual relative to its orthogonal component. The proposed approach, whose implementation is open-source, achieves state-of-the-art results on the public benchmarks available at \url{ann-benchmarks.com}.
\end{abstract}

\section{Introduction} \label{sec:intro}
Maximum inner product search (MIPS) has become a popular paradigm for solving large scale classification and retrieval tasks.
For example, in recommendation systems, user queries and documents are embedded into a dense vector space of the same dimensionality and MIPS is used to find the most relevant documents given a user query~\citep{MIPSRecSys}.
Similarly, in extreme classification tasks~\citep{MIPSForEC}, MIPS is used to predict the class label when a large number of classes, often on the order of millions or even billions are involved.
Lately, MIPS has also been applied to training tasks such as scalable gradient computation in large output spaces~\citep{LossDecompHsu}, efficient sampling for speeding up softmax computation~\citep{MIPSSampledSoftmax} and sparse updates in end-to-end trainable memory systems~\citep{NeuralEpisodicControl}.

To formally define the Maximum Inner Product Search (MIPS) problem, consider a database $X = \{x_i\}_{i=1,2,\dots,n}$ with $n$ datapoints, where each datapoint $x_i \in \R^d$ in a $d$-dimensional vector space.
In the MIPS setup, given a query $q \in \R^d$, we would like to find the datapoint $x \in X$ that has the highest inner product with $q$, i.e., we would like to identify
\begin{align*}
\vspace{-0.2inch}
     x_i^* := \argmax_{x_i \in X} \inner{q, x_i}.
\vspace{-0.2inch}
 \end{align*}

\vspace{-0.05in}
Exhaustively computing the exact inner product between $q$ and $n$ datapoints is often expensive and sometimes infeasible. Several techniques have been proposed in the literature based on hashing, graph search, or quantization to solve the approximate maximum inner product search problem efficiently, and the quantization based techniques have shown strong performance~\citep{OPQ,AQ,FAISS}.

In most traditional quantization works, the objective in the quantization procedures is to minimize the reconstruction error for the database points. We show this is a suboptimal loss function for MIPS. This is because for a given query, quantization error for database points that score higher, or have larger inner products, is more important. Using this intuition, we propose a new family of score-aware quantization loss functions and apply it to multiple quantization techniques. Our contributions are as follows:
\begin{itemize}
    \item We propose the score-aware quantization loss function. The proposed loss can work under any weighting function of the inner product and regardless of whether the datapoints vary in norm.
    \item Under natural statistical assumptions, we show that the score-aware quantization loss can be efficiently calculated. The loss function leads to an anisotropic weighting that more greatly penalizes error parallel with the datapoint than error orthogonal to the datapoint.
    \item The proposed loss is generally applicable to many quantization methods.
    We demonstrate the codebook learning and quantization procedures for product quantization and vector quantization can be efficiently adapted to the proposed loss function.
    \item We show that anisotropic quantization leads to large MIPS performance gains over reconstruction loss-based techniques. Our method achieves state-of-the-art performance on standard large-scale benchmarks such as Glove-1.2M. In addition to recall gains, anisotropic quantization gives significantly more accurate inner product value approximations.
\end{itemize}

\section{Background and Related Works} \label{sec:related}
\subsection{Inference as Maximum Inner Product Search}
Efficient maximum inner product search (MIPS) is necessary for many large-scale machine learning systems. One popular approach to information retrieval systems and recommender systems uses \textit{representation learning} in the embedding space. In this framework, we learn embedding functions to map items to be retrieved in a common vector space, where the items can be words, images, users, audio, products, web pages, graph nodes, or anything of interest \citep{MIPSRecSys, weston2010large, guo2016deep, gillick2019learning, wu2017starspace}.

In recommender systems, two networks are jointly trained to generate query (user) vectors and item vectors, such that embedding vectors of queries and relevant items have high inner product when computed in the embedding space. To perform inference, we first pre-compute a database of embedding vectors for items to be recommended. When a query arrives, we compute the query embedding then return the items with the highest inner product. In extreme classification, a neural network classifier is trained, where each row of the weight matrix of the classification layer corresponds to the embedding of a class label~\citep{MIPSForEC, reddi2019stochastic}. In both settings, the computationally expensive operation is finding the item embedding that has the largest inner product with the query embedding, which can be efficiently solved by Maximum Inner Product Search (MIPS).

\subsection{Methods for accelerating MIPS}
There is a large body of similarity search literature on max inner product and nearest neighbor search. We refer readers to~\citep{JingdongSurvey,SanjivSurvey} for a comprehensive survey. We include a brief summary here.

There are two main tasks required to develop an efficient MIPS system. One task is to reduce the number of items that are scored to identify the top result. This is typically done with a space partitioning method. The other task is improving the rate at which items are scored. This is typically done with quantization, and is where the main contribution of our work lies. Successful implementation of MIPS systems requires good performance in both tasks.

Many researchers have developed high quality implementations of libraries for nearest neighbor search, such as SPTAG~\cite{sptag}, FAISS~\cite{FAISS}, and hnswlib~\cite{hnsw}. We compare with the ones available on ANN-Benchmarks in Section~\ref{sec:experiments}.

\subsubsection{Reducing the Number of Evaluations}
One class of approaches to reducing the number of items scored is space partitioning. These approaches partition the space into different buckets. To perform MIPS in this setting, we find the relevant buckets for a given query and score only the items in these buckets.

Examples of this approach include tree search methods and locality sensitive hashing. Tree search methods such as \citep{FLANN, RPTree} partition the space recursively, forming a tree. Locality sensitive hashing \citep{ALSH, neyshabur2014symmetric, LSH, FALCONN} partitions the space using a similarity-preserving hash function. There is also a class of approaches based on graph search \citep{hnsw, FANNG}. These methods work by navigating a graph by greedily selecting the neighbor with the highest dot product.

\subsubsection{Quantization}
Quantization is an important technique for building state-of-the-art MIPS systems in large scale settings. Below we describe the several ways that quantization improves performance.
\begin{itemize}
    \item Efficient dot product computations: We can calculate the dot product of a $d$ dimensional query vector with $n$ quantized points in time $O(dk + mn)$ using look up tables, where $k$ is the size of each quantization codebook and $m$ is the number of codebooks. For typical choices of $k$ and $m$ this is faster than the $O(n d)$ complexity required for exact computation.
    \item Memory bandwidth: modern processors need workloads with a high amount of computation per memory read in order to fully utilize their resources. Quantization compresses datapoints, resulting in less memory bandwidth usage and higher processor utilization.
    \item Storage: quantized datapoints take up less space in memory or on disk. For large-scale datasets, this allows more datapoints to be stored on a single machine.
\end{itemize}

One approach to quantization is with random projections \citep{charikar2002similarity, vempala2005random, li2019random}. One issue with random projections is that quantization is oblivious to the data, and it may be more efficient to use a quantization method that is able to exploit structure in the data. Quantization methods of this form are available for binary quantization \citep{KmeansHash, deephashing, stochastic_hashing}, product quantization \citep{PQ, QUIPS, zhang2014composite,MSQ}, additive quantization \citep{AQ, LSQ}, and ternary quantization \citep{zhu2016trained}. We discuss product quantization in more detail in Section~\ref{sec:quantization}. There are also lines of work that focus on learning transformations before quantization~\citep{ITQ,OPQ, Sablayrolles18}. Learning quantization from the observed data distribution also has been studied in \cite{Marcheret09, Morozov_2019_ICCV, babenko16}.

Our work differs from the above methods as they essentially focus on minimizing reconstruction error as a loss function, while we develop an approach in the following section where we minimize a novel loss function that is designed to improve the downstream MIPS objective.

We also highlight the work \citet{may2019downstream}, where they consider quantization objectives for word embeddings that improve the downstream performance of training models for natural language processing tasks.

\section{Problem Formulation} \label{sec:formulation}
Common quantization techniques focus on minimizing the reconstruction error (sum of squared error) when $x$ is quantized to $\quant{x}$. It can be shown that minimizing the reconstruction errors is equivalent to minimizing the expected inner product quantization error under a mild condition on the query distribution without assumption on the database point distribution. Indeed, consider the quantization objective of minimizing the expected total inner product quantization errors over the query distribution:
\begin{align}
   \mathbb{E}_q\sum_{i=1}^n (\inner{q, x_i} - \inner{q, \quant{x_i}})^2= \mathbb{E}_q \sum_{i=1}^n \inner{q, x_i - \quant{x_i}}^2. \label{eq:quant_objective}
\end{align}

Under the assumption that $q$ is isotropic, i.e., $\mathbb{E} [q q^T] = c I$, where $I$ is the identity matrix and $c \in \R^{+}$, the objective function becomes 
\begin{align*}
\sum_{i=1}^n \mathbb{E}_q \inner{q, x_i - \quant{x_i}}^2 &= \sum_{i=1}^n \mathbb{E}_q (x_i - \quant{x_i})^T q q^T  (x_i - \quant{x_i}) \\
&= c \sum_{i=1}^n  \|x_i - \quant{x_i} \|^2
\end{align*}
Therefore, the objective becomes minimizing the reconstruction errors of the database points $\sum_{i=1}^n \|x_i-\quant{x_i} \|^2$, and this has been considered extensively in the literature.

One key observation about the above objective function~\eqref{eq:quant_objective} is that it takes expectation over all possible combinations of datapoints $x$ and queries $q$. However, it is easy to see that not all pairs of $(x,q)$ are equally important. The approximation error on the pairs which have a high inner product is far more important since they are likely to be among the top ranked pairs and can greatly affect the search result, while for the pairs whose inner product is low the approximation error matters much less. In other words, for a given datapoint $x$, we should quantize it with a bigger focus on its error with those queries which have high inner product with $x$. See Figure \ref{fig:intuition} for the illustration.

Following this key observation, we propose the \textit{score-aware quantization loss}. This is a new loss function for quantization that weighs the inner product approximation error by $w$, an arbitrary function of our choice that returns a weight based on the value of the true inner product. Specifically, we define the loss function as the following:

\begin{definition}
Given a datapoint $x_i$, its quantization $\quant{x_i}$, and a weight function $w: \mathbb{R} \mapsto \mathbb{R}^+$ of the inner product score, the  \textit{score-aware quantization} loss with respect to a query distribution $\mathcal{Q}$ is defined as
\begin{equation}
\ell(x_i, \quant{x_i}, w) = \mathbb{E}_{q \sim \mathcal{Q}}[w(\inner{q, x_i}) \inner{q, x_i-\quant{x_i}}^2].
\label{eq:loss_function}    
\end{equation}

\end{definition}

\begin{figure*}
\begin{subfigure}[b]{0.35\textwidth}
\includegraphics[width=1 \textwidth]{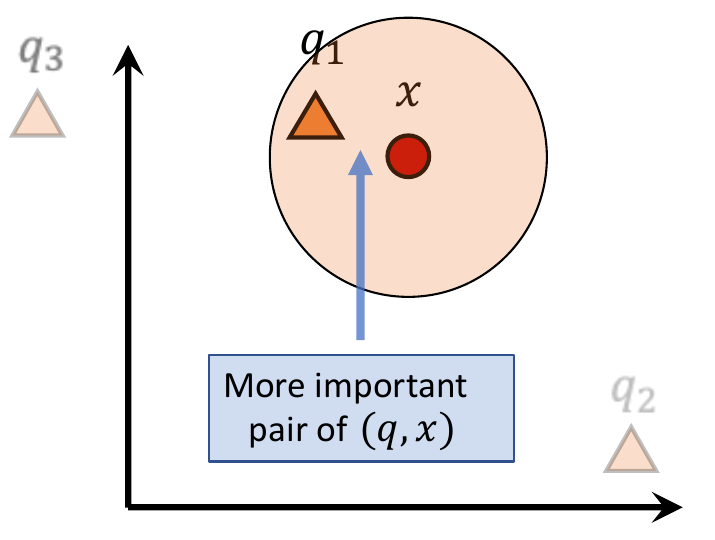}
\subcaption{\label{fig:intuition1}}
\end{subfigure}
\begin{subfigure}[b]{0.32\textwidth}
\includegraphics[width=1 \textwidth]{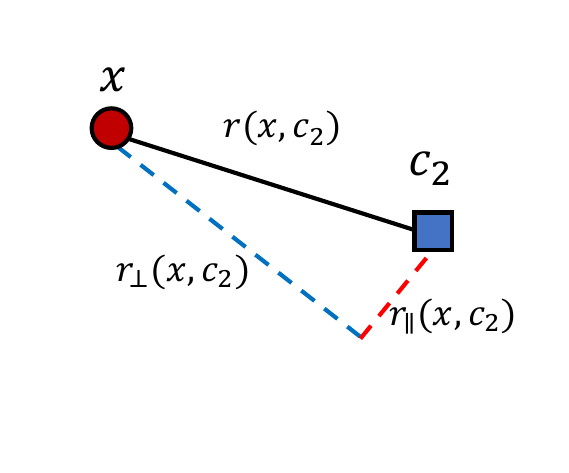}
\subcaption{\label{fig:quant_notation}}
\end{subfigure}
\begin{subfigure}[b]{0.32\textwidth}
\includegraphics[width=1 \textwidth]{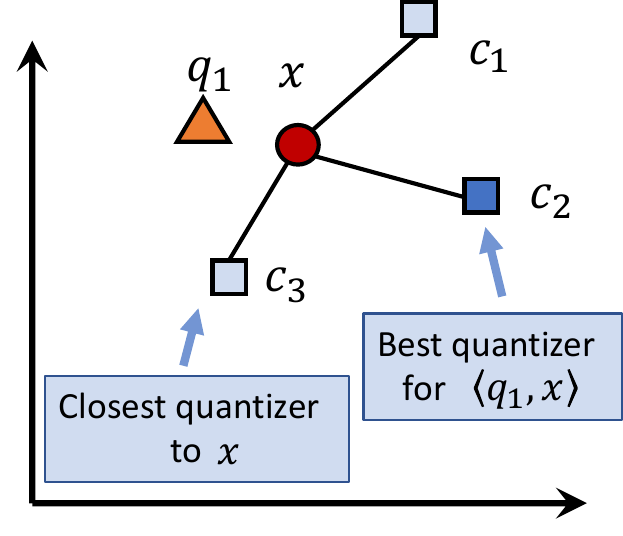}
\subcaption{\label{fig:intuition2}}
\end{subfigure}   
\caption{(a) Not all pairs of $q$ and $x$ are equally important: for $x$, it is more important to accurately quantize the inner product of $\inner{q_1, x}$ than $\inner{q_2, x}$ or $\inner{q_3, x}$, because $\inner{q_1, x}$ has a higher inner product and thus is more likely to be the maximum; (b) Quantization error of $x$ given one of its quantizer $c_2$ can be decomposed to a parallel component $r_{\parallel}$ and an orthogonal component $r_{\perp}$. (c) Graphical illustration of the intuition behind Equation~\eqref{eq:loss_function}. Even if $c_3$ is closer to $x$ in terms of Euclidean distance, $c_2$ is a better quantizer than $c_3$ in terms of inner product approximation error of $\inner{q_1, x-c}$. Notice that $c_3$ incur more parallel loss ($r_{\parallel}$), while $c_2$ incur more orthogonal loss ($r_{\perp}$).}\label{fig:intuition}
\end{figure*}

Since the norm of $q$ does not matter to the ranking result, we can assume $||q||=1$ without loss of generality. Similarly, assuming we have no prior knowledge of the query distribution $\mathcal{Q}$, we trivially assume $q$ is uniformly spherically distributed. The expectation can be recomputed if $\mathcal{Q}$ is known or estimated empirically.

\subsection{Analyzing Score-Aware Quantization Loss}\label{sec:analyze_general_w}
We show that regardless of the choice of $w$, a score-aware quantization loss $\ell(x_i, \quant{x_i}, w)$ always decomposes into an anisotropic weighted sum of the magnitudes of the parallel and orthogonal residual errors. These two errors are defined as follows: first, define the residual error of a quantization $\quant{x_i}$ as $x_i-\quant{x_i}$. The \textit{parallel residual error} is the component of the residual error parallel to the datapoint $x_i$; it can be computed as

$$r_{\parallel}(x_i,\quant{x_i})=\dfrac{\inner{(x_i-\quant{x_i}),x_i}x_i}{||x_i||^2}.$$

\textit{Orthogonal residual error} is defined analogously, and can be computed as

$$r_\perp(x_i, \quant{x}_i) = (x_i - \quant{x}_i) - r_\parallel(x_i, \quant{x}_i).$$

These two components are illustrated in Figure~\ref{fig:quant_notation}. The relative weights of these two error components in contributing to the score-aware loss are determined by the choice of $w$.
\begin{theorem}\label{thm:general_w}
Suppose we are given a datapoint $x_i$, its quantization $\quant{x_i}$, and a weight function $w$. Assuming that query $q$ is uniformly distributed in the $d$-dimensional unit sphere, the score-aware quantization loss equals

$$\begin{aligned}
\ell(x_i, \quant{x_i}, w) &= h_{\parallel}(w,||x_i||)||r_{\parallel}(x_i,\quant{x_i})||^2 \\
&+ h_{\perp}(w,||x_i||)||r_{\perp}(x_i,\quant{x_i})||^2
\end{aligned}$$

with $h_{\parallel}$ and $h_{\perp}$ defined as follows:
$$\begin{aligned}
h_{\parallel} &:= (d-1)\int_0^\pi w(||x_i||\cos\theta)(\sin^{d-2}\theta-\sin^{d}\theta) d\theta \\
h_{\perp} &:= \int_0^\pi w(||x_i||\cos\theta)\sin^d\theta d\theta.
\end{aligned}$$

\end{theorem}
\begin{proof}
See Appendix Section~\ref{proof:general_w}.
\end{proof}

Any weight function would work for the above proposed loss. For the MIPS problem, it is intuitive to choose $w$ so that it puts greater weight on larger inner products. For such $w$, we show that parallel quantization error is weighted more heavily than orthogonal quantization error. This is formalized below and illustrated in Figure~\ref{fig:intuition}.

\begin{theorem}\label{thm:parallel_greater}
For any $w$ such that $w(t)=0$ for $t<0$ and $w(t)$ is monotonically non-decreasing for $t\ge0$,

$$h_{\parallel}(w, ||x_i||)\ge h_{\perp}(w, ||x_i||)$$

with equality if and only if $w(t)$ is constant for $t\in[-||x_i||, ||x_i||].$
\end{theorem}
\begin{proof}
See Appendix Section~\ref{proof:parallel_greater}.
\end{proof}

\subsection{Special case of $w(t)=\mathbf{I}(t\ge T)$}
\label{sec:indicator}
One particular $w$ of interest is the function $w(t)=\mathbf{I}(t\ge T)$. This weight function only considers quantization loss when the dot product is above a threshold $T$. Since $\mathbf{I}(t\ge T)$ satisfies the conditions for Theorem \ref{thm:parallel_greater}, it effectively penalizes parallel quantization error more greatly than orthogonal error. With this weight function, our expressions for $h_{\parallel}$ and $h_{\perp}$ simplify to:
$$\begin{aligned}
h_{\parallel} &= (d-1)\int_0^{\arccos (T/||x_i||)} \sin^{d-2}\theta-\sin^{d}\theta d\theta \\
h_{\perp} &= \int_0^{\arccos (T/||x_i||)} \sin^d\theta d\theta
\end{aligned}$$

With $w(t)=\mathbf{I}(t\ge T)$, we have
\begin{align*}
\ell(x_i, \quant{x}_i, w) =\;& h_{\parallel}(w, ||x_i||)||r_{\parallel}(x_i,\quant{x}_i)||^2 + \\
                        & h_{\perp}(w, ||x_i||) ||r_{\perp}(x_i,\quant{x}_i)||^2 \\
\propto\; &\eta(w, ||x_i||)||r_{\parallel}(x_i,\quant{x}_i)||^2 + ||r_{\perp}(x_i,\quant{x}_i)||^2
\end{align*}
where $\eta(w, ||x_i||) := \dfrac{h_{\parallel}(w, ||x_i||)}{h_{\perp}(w, ||x_i||)}$.

\begin{figure}
    \centering
    \includegraphics[width=0.8\linewidth]{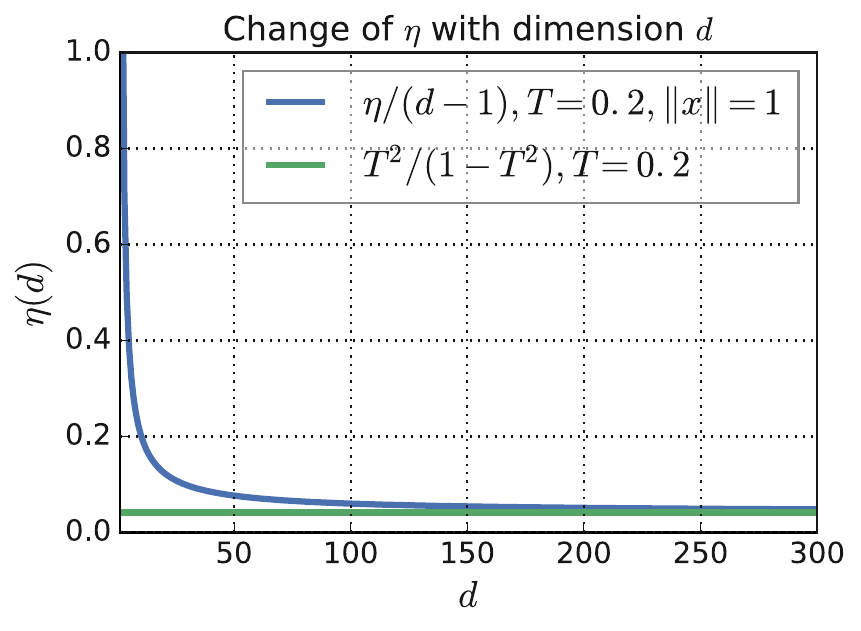}
    \caption{The ratio $\eta(\mathbf{I}(t\ge T=0.2), \|x\|=1)/(d-1)$ in Theorem \ref{thm:eta-limit} computed analytically as a function of $d$ quickly approaches the limit defined in Equation \eqref{eq:eta-limit}.}
    \label{fig:lambda}
\end{figure}

We can recursively compute $\eta(w=\mathbf{I}(t\ge T), ||x_i||)$ as a function of $d$ analytically. Furthermore we can prove that $\frac{\eta}{d-1}$ has an limit as $d\to\infty$, as demonstrated empirically in Figure~\ref{fig:lambda}. We can use this limit, which is easy to evaluate, as a proxy of $\eta$ in computing the proposed loss.
\begin{theorem}\label{thm:eta-limit}
\begin{equation}\label{eq:eta-limit}
    \lim_{d\to\infty}\dfrac{\eta(\mathbf{I}(t\ge T), ||x_i||)}{d-1}=\dfrac{(T/||x_i||)^2}{1-(T/||x_i||)^2}
\end{equation}
\end{theorem}
\begin{proof}
See Appendix Section~\ref{proof:normalized-dataset-results}.
\end{proof}

As special cases, when $T=0$, $\eta(\mathbf{I}(t\ge 0), ||x_i||)=1$ which implies parallel and orthogonal errors are weighted equally. When $T=\||x_i||$, we have $\eta(\mathbf{I}(t\ge ||x_i||), ||x_i||)=\infty$ which indicates we should only consider parallel error.

Theorem~\ref{thm:general_w} shows that the weight of each datapoint's parallel and orthogonal quantization errors are dependent on $||x_i||$. However, when the database has constant norm, i.e. $||x_i|| = c$, we can use the following simplified form:
\begin{align*}
&\sum_{i=1}^{n} \ell(x_i, \quant{x}_i, \mathbf{I}(t\ge T)) \\
&\propto \eta(w,c) \sum_{i=1}^n ||r_{\parallel}(x_i, \quant{x}_i)||^2 + \sum_{i=1}^n ||r_{\perp}(x_i,\quant{x}_i)||^2
\end{align*}

\section{Application to Quantization Techniques}\label{sec:quantization}

In this section we consider the codebook learning and quantization procedure for our proposed anisotropic loss function. In the previous sections, we established that the loss function, $\ell(x_i, \quant{x}_i, w)$ leads to a weighted combination of parallel quantization error and orthogonal quantization error. In practice, we can choose a fixed $\eta$ according to the choice of $w$ such as the one suggested in Section~\ref{sec:indicator}.

In vector quantization, we first construct a dictionary $C = \{c_1, c_2, \ldots, c_k\}$. To quantize a vector $x$ we replace $x$ with one of the codewords. Typically, the quantized vector $\tilde{x}$ minimizes some loss function: $\quant{x} = \argmin_{c_1, c_2, \ldots, c_k} L(x_i, c_i)$.

After we quantize a database of $n$ points, we can calculate the dot product of a query vector $q$ with all quantized points in $O(k d + n)$ time. This is much better than the $O(n d)$ time required for the original unquantized database. We achieve the $O(kd+n)$ runtime by computing a lookup table containing the inner product of the $q$ with each of the $k$ codewords in $O(k d)$ time. We then do a table lookup for each of the $n$ datapoints to get their corresponding inner products.

In order to construct the dictionary $C$, we need to optimize the choice of codewords over the loss function. For $\ell_2$-reconstruction loss, the optimization problem becomes
\[
\min_{c_1, c_2, \ldots, c_k \in \mathbb{R}^d}\sum_{x_i}\min_{\quant{x}_i \in \{c_1, c_2, \ldots, c_k\}}\|x_i - \quant{x}_i\|^2.
\]
This is exactly the well-studied $k$-means clustering objective, which is often solved using  Lloyd's algorithm.

If, as in the previous section, we have our loss function $\ell(x, \tilde{x}) = h_{i, \parallel} \|r_\parallel(x_i, \quant{x_i})\|^2 + h_{i, \perp}\|r_\perp(x_i, \quant{x_i})\|^2$ for appropriate scaling parameters $h_{i, \parallel}$, $h_{i, \perp}$, we obtain a new objective function we call the \textit{anisotropic vector quantization problem}.

\begin{definition}
Given a dataset $x_1, x_2, \ldots, x_n$ of points in $\mathbb{R}^d$, scaling parameters $h_{i, \parallel}$, $h_{i, \perp}$ for every datapoint $x_i$, and $k$ codewords, the \textit{anisotropic vector quantization problem} is finding the $k$ codewords that minimize the objective function 
\begin{align*}
\min_{c_1, \ldots, c_k}\sum_{x_i}\min_{\quant{x}_i \in \{c_1, \ldots, c_k\}} &h_{i, \parallel}\|r_\parallel(x_i, \tilde{x}_i)\|^2 \\
+ &h_{i, \perp}\|r_\perp(x_i, \tilde{x}_i)\|^2.
\end{align*}
\end{definition}

Next we develop an iterative algorithm to optimize the anisotropic vector quantization problem. Similar to Lloyd's algorithm \citep{lloyd1982least}, our algorithm iterate between partition assignment step and codebook update step:

\begin{enumerate}
    \item (Initialization Step) Initialize codewords $c_1, c_2, \ldots, c_k$ to be random datapoints sampled from $x_1 \ldots x_n$.
    \item (Partition Assignment Step) For each datapoint $x_i$ find its codeword $\quant{x_i} = \argmin_{\quant{x}_i \in \{c_1, \ldots, c_k\}} \ell(x_i, \quant{x}_i)$. This can be done by enumerating all $k$ possile choices of codewords.
    \item (Codebook Update Step) For every codeword $c_j$, let $X_j$ be all datapoints $x_i$ such that $\quant{x}_i = c_j$. Update $c_j$ by
    \[c_j \gets \argmin_{c \in \mathbb{R}^d}\sum_{x_i \in X_j} \ell(x_i, c).\]
    \item Repeat Step 2 and Step 3 until convergence to a fixed point or maximum number of iteration is reached.
\end{enumerate}

In each iteration, we need perform update step for each of the codeword. Given a partition of the datapoints $X_j$, we can find the optimal value of the codeword $c_j$ that minimizes the following objective:
\begin{equation}\label{eqn:anisotropic-objective}
    c_j = \argmin_{c \in \mathbb{R}^d}\sum_{x \in X_j} h_{i, \parallel} \|r_\parallel(x_i, c)\|^2 + h_{i, \perp}\|r_\perp(x_i, c)\|^2.
\end{equation}

By setting gradient respect to $c_j$ to zero, we obtain the following update rule:

\begin{theorem}\label{thm:anisotropic-update}
Optimal codeword $c_j$ can be obtained in closed form by solving the optimization problem in Equation \eqref{eqn:anisotropic-objective} for a partition $X_j$. The update rule for the codebook is
$$\begin{aligned}
c_j^* = \Bigg(I&\sum_{x_i\in X_j}h_{i, \perp} + \\
&\sum_{x_i \in X_j} \dfrac{h_{i, \parallel}-h_{i, \perp}}{||x_i||^2}x_ix_i^T\Bigg)^{-1}\sum_{x_i \in X_j}h_{i, \parallel} x_i
\end{aligned}$$
\end{theorem}
\begin{proof}
See Section \ref{proof:anisotropic-update} of the Appendix for the proof.
\end{proof}

As expected, we see that when all $h_{i, \parallel} = h_{i, \perp}$, our codeword update is equivalent to finding the weighted average of the partition. Furthermore, if $w(t)=1$, the update rule becomes finding the average of datapoints in the partition, same as standard $k$-means update rule. Additionally, since there are only a finite number of partitions and at every iteration the loss function decreases or stays constant, our solution will eventually converge to a fixed point.

\subsection{Product Quantization}

In vector quantization with a dictionary of size $k$, we quantize each datapoint into one of $k$ possible codewords. We can think of this as encoding each datapoint with one dimension with $k$ possible states.

With \textit{product quantization} we encode each datapoint into an $M$ dimensional codeword, each with $k$ possible states. This allows us to represent $k^M$ possible codewords, which would not be scalable with vector quantization. To do this, we split each datapoint $x$ into $M$ subspaces each of dimension $d / M$: $x = (x^{(1)}, x^{(2)}, \ldots, x^{(m)})$. We then create $M$ dictionaries $C^{(1)}, C^{(2)}, \ldots, C^{(m)}$, each with $k$ codewords of dimension $d / M$. Each datapoint would then be encoded with $M$ dimensions, with every dimension taking one of $k$ states.

To calculate distances with product quantization, for every dictionary $C^{(m)}$ we calculate the partial dot product of the relevant subspace of the query with every codeword in the dictionary. The final dot product is obtain by sum up all $M$ partial dot product. We can then calculate the dot product with $m$ quantized datapoints in time $O(kd + m n)$.

Using our anisotropic loss function $\ell(x_i, \quant{x}_i) = h_{i, \parallel} \| r_\parallel(x_i, \quant{x}_i)\|^2 + h_{i, \perp}\|r_\perp(x_i, \quant{x}_i)\|^2$ we obtain a new objective function for product quantization we call the anisotropic product quantization problem.
\begin{definition}
Given a dataset $x_1, x_2, \ldots, x_n$ of points in $\mathbb{R}^d$, a scaling parameter $\eta$, a number $M$ of dictionaries each with elements of size $d / M$ and $k$ codewords in each dictionary, the \textit{anisotropic product quantization problem} is to find the $M$ dictionaries that minimizes
\begin{align*}
\min_{\substack{C^{(m)} \subseteq \mathbb{R}^{d / M} \\ \vert C^{(m)} \vert = k}}\sum_{x_i}\min_{\quant{x}_i \in \prod_m C^{(m)}} &h_{i, \parallel} \|r_\parallel(x_i, \quant{x}_i)\|^2 \\
&+ h_{i, \perp}\|r_\perp(x_i, \quant{x}_i)\|^2.
\end{align*}
\end{definition}

We again consider an iterative algorithm for the problem. We first initialize all quantized datapoints with some element from every dictionary. We then consider the following iterative procedure:
\begin{enumerate}
    \item (Initialization Step) Select a dictionary $C^{(m)}$ by sampling from $\{x^{(m)}_1, \ldots x^{(m)}_n\}$.
    \item (Partition Assignment Step) For each datapoint $x_i$, update $\quant{x}_i$ by using the value of $c \in C^{(m)}$ that minimizes the anisotropic loss of $\quant{x}_i$.
    \item (Codebook Update Step) Optimize the loss function over all codewords in all dictionaries while keeping every dictionaries partitions constant.
    \item Repeat Step 2 and Step 3 until convergence to a fixed point or maximum number of iteration is reached.
\end{enumerate}

We can perform the update step efficiently since once the partitions are fixed the update step minimizes a convex loss, similar to that of vector quantization. We include details in Section~\ref{sec:pq_update} of the Appendix. Additionally, since there are a finite number of partition assignment and at every step the loss function decreases or stays constant, our solution will eventually converge to a fixed point. We note that we can also optionally initialize the codebook by first training the codebook under regular $\ell_2$-reconstruction loss, which speed up training process.

\section{Experiments} \label{sec:experiments}

In this section, we show our proposed quantization objective leads to improved performance on maximum inner product search.
First, we fix the quantization mechanism and compare traditional reconstruction loss with our proposed loss to show that score-aware loss
leads to better retrieval performance and more accurate estimation of maximum inner product values.
Next, we compare in fixed-bit-rate settings against QUIPS and LSQ, which are the current state-of-the-art for many MIPS tasks.
Finally, we analyze the end-to-end MIPS retrieval performance of our algorithm in terms of its speed-recall trade-off in a standardized hardware environment.
We used the benchmark setup from \url{ann-benchmarks.com}, which provides 11 competitive baselines with pre-tuned parameters.
We plot each algorithm's speed-recall curve and show ours achieves the state-of-the-art.

\renewcommand\floatpagefraction{.8}
\begin{figure}
\centering
\begin{subfigure}[b]{\linewidth}
    \centering
    \parbox[b]{.1\linewidth}{
        \subcaption{\label{fig:lambda_sensitivity}}
        \vspace{40mm}
    }\parbox[b]{.9\linewidth}{
        \includegraphics[width=0.95 \linewidth]{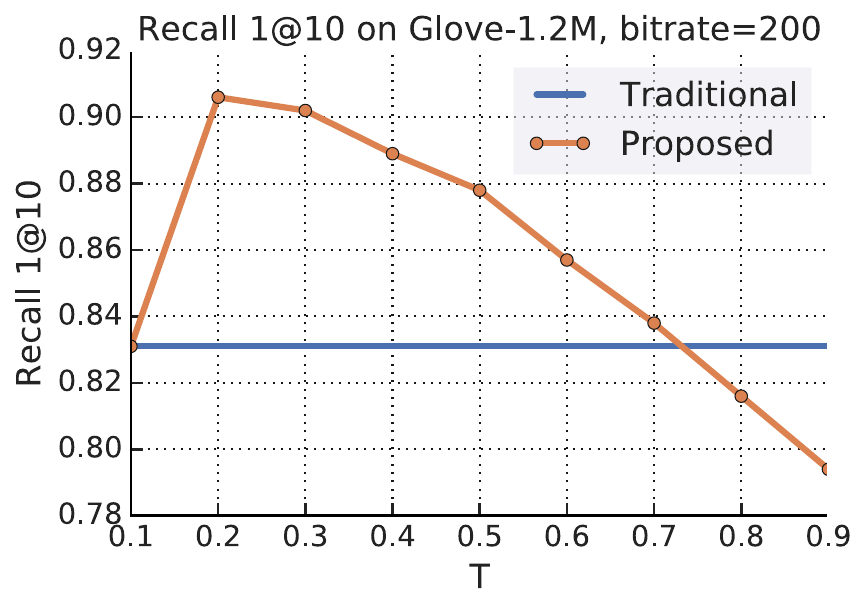}
    }
\end{subfigure}
\begin{subfigure}[b]{\linewidth}
    \centering
    \parbox[b]{.1\linewidth}{
        \subcaption{\label{fig:relative_error}}
        \vspace{40mm}
    }\parbox[b]{.9\linewidth}{
        \includegraphics[width=0.95\linewidth]{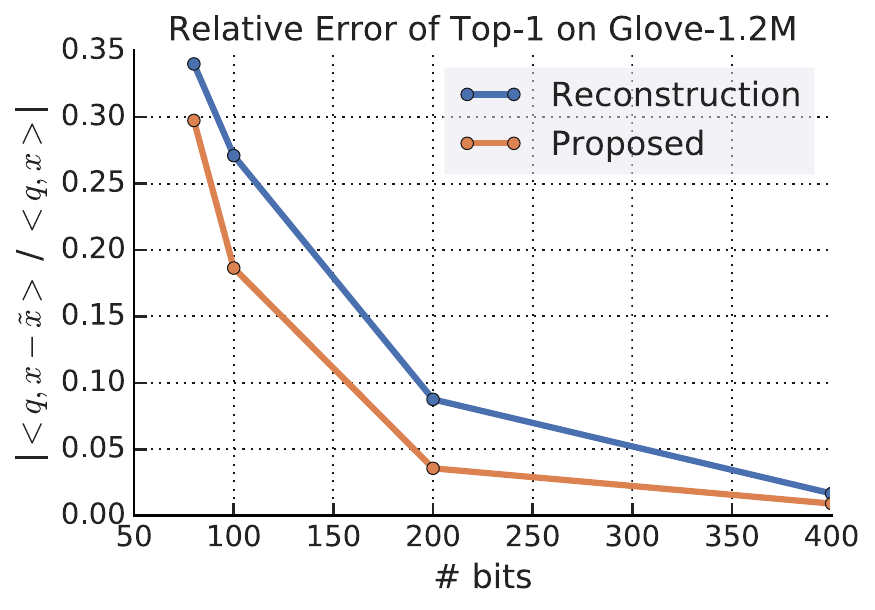}
    }
\end{subfigure}
\caption{(a) The retrieval Recall1@10 for different values of the threshold $T$. We see that for $T = 0.2$ (corresponding to $\eta = 4.125$) our proposed score-aware quantization loss achieves significantly better Recall than traditional reconstruction loss. (b) The relative error of inner product estimation for true Top-1 on \texttt{Glove1.2M} dataset across multiple number of bits settings. We see that our proposed score-aware quantization loss reduces the relative error compared to reconstruction loss.}
\end{figure}

\begin{figure*}[ht]
\begin{minipage}{0.4 \textwidth}
\includegraphics[width=0.94 \linewidth]{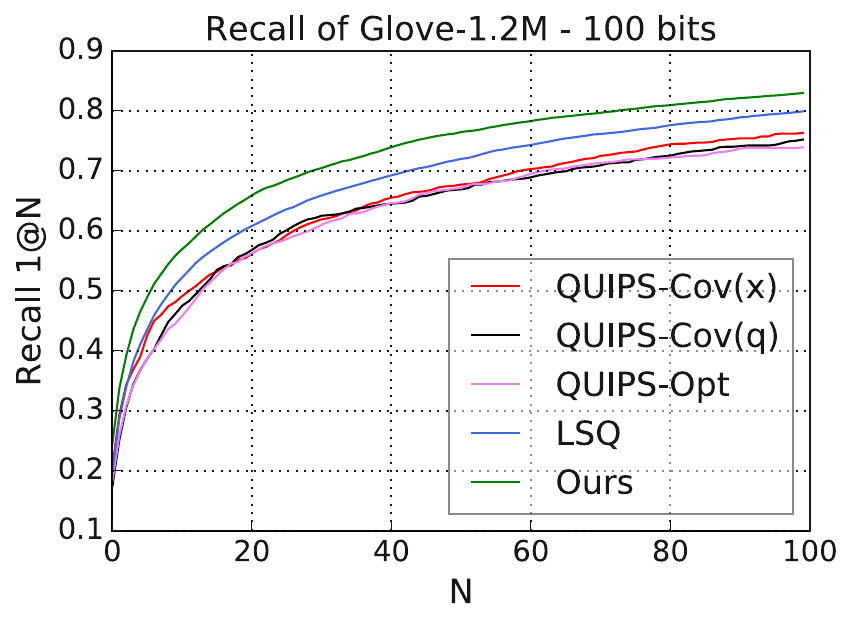} \\
\includegraphics[width=0.94 \linewidth]{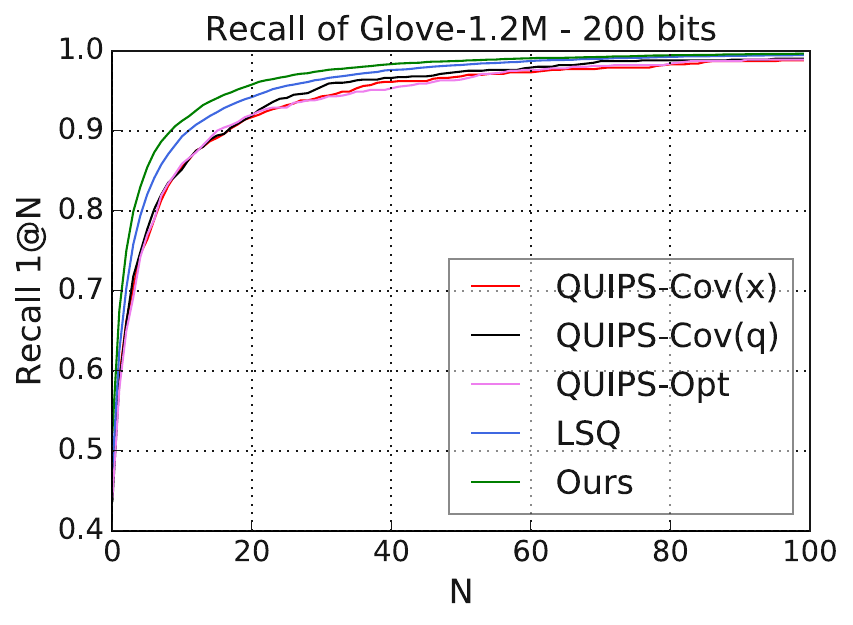}

\centering
\subcaption{\label{fig:fixedbitrate} MIPS recall on \texttt{Glove1.2M}.}
\end{minipage}
\begin{minipage}{0.56 \textwidth}
\centering
\includegraphics[width= \linewidth]{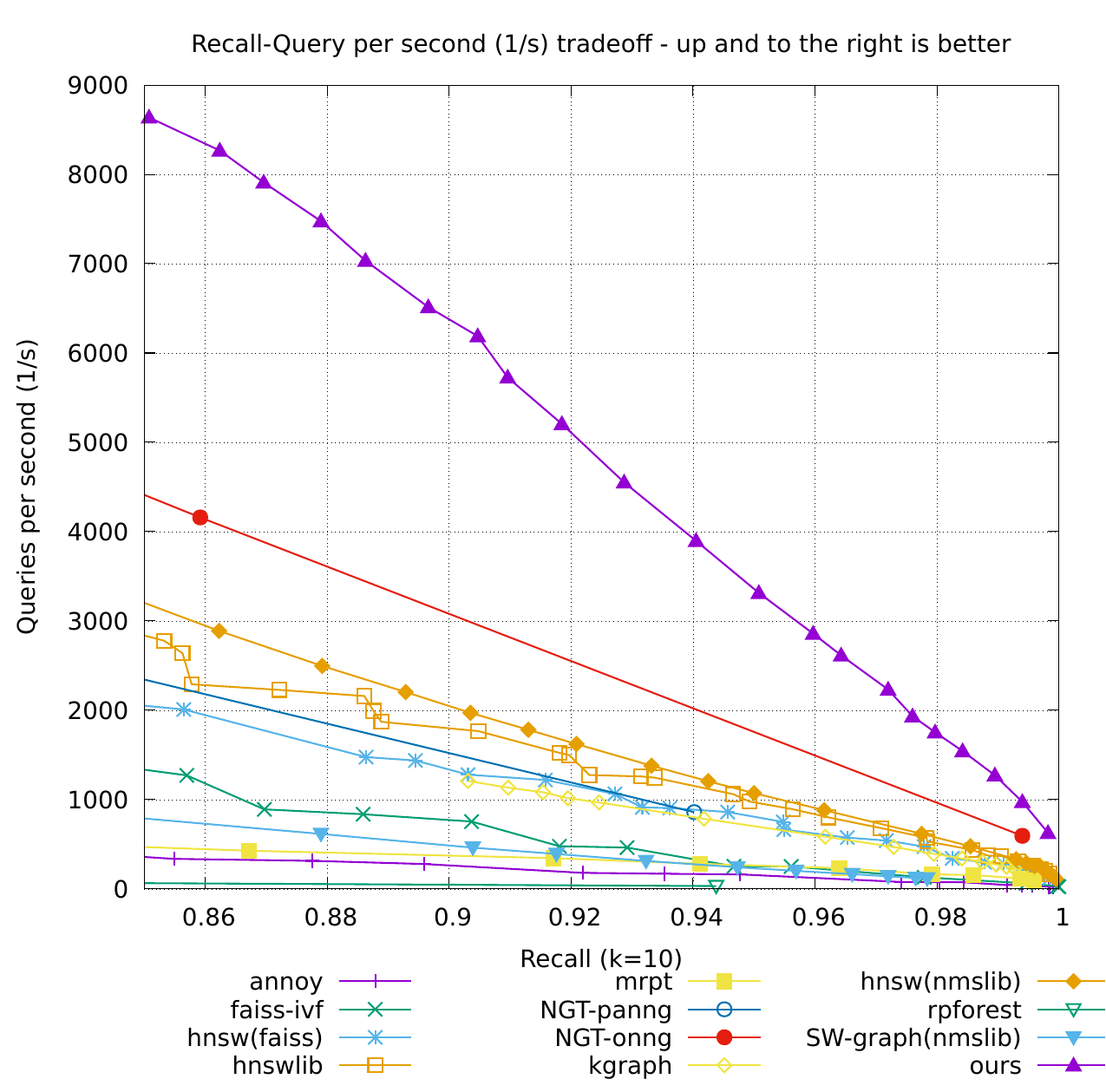}
\subcaption{\label{fig:ann-benchmarks}Speed-recall trade-off on \texttt{Glove1.2M} Recall 10@10.}
\end{minipage}

\caption{\label{fig:recall} (a) Recall 1@N curve on \texttt{Glove1.2M} comparing with variants of QUIPS~\cite{QUIPS} and LSQ \cite{LSQ} on MIPS tasks. We see that our method improves over all of these methods. (b) Recall-Speed benchmark with 11 baselines from~\cite{aumuller2019ann} on \texttt{Glove1.2M}. The parameters of each baseline are pre-tuned and released on: \url{http://ann-benchmarks.com/}. We see that our approach is the fastest in the high recall regime.}
\end{figure*}

\subsection{Direct comparison with reconstruction loss}
We compare our proposed score-aware quantization loss with the traditional reconstruction loss by fixing all parameters other than the loss function in the following experiments.

We use \texttt{Glove1.2M} which is a collection of 1.2 million 100-dimensional word embeddings trained as described in \cite{PenningtonGlove}. See Section~\ref{sec:idiosyncrasy} of the Appendix for our rationale for choosing this dataset. For all experiments we choose $w(t)=\mathbf{I}(t\ge T)$. The Glove dataset is meant to be used with a cosine distance similarity metric, while our algorithm is designed for the more general MIPS task. MIPS is equivalent to cosine similarity search when all datapoints are equal-norm, so we adopt our technique to cosine similarity search by unit-normalizing all datapoints at training time.

We first compare the two losses by their Recall1@10 when used for product quantization on \texttt{Glove1.2M}, as shown in Figure.~\ref{fig:lambda_sensitivity}. We learn a dictionary by optimizing product quantization with reconstruction loss. We then quantize datapoints two ways, first by minimizing reconstruction loss and then by minimizing score-aware loss. We see that score-aware quantization loss achieves significant recall gains as long as $T$ is chosen reasonably. For all subsequent experiments, we set $T = 0.2$, which by the limit in Equation \eqref{eq:eta-limit} corresponds to a value of $\eta = 4.125$.

Next we look at the accuracy of the estimated top-1 inner product as measured by relative error: $|\frac{\inner{q, x} - \inner{q, \quant{x}}}{\inner{q, x}}|$. This is important in application scenarios where an accurate estimate of $\inner{q, x}$ is needed, such as softmax approximation, where the inner product values are often logits later used to compute probabilities. One direct consequence of score-aware loss functions is that the objective weighs pairs by their importance and thus leads to lower estimation error on top-ranking pairs. We see in Figure.~\ref{fig:relative_error} that our score-aware loss leads to smaller relative error over all bitrate settings.

Datasets other than Glove demonstrate similar performance gains from score-aware quantization loss. See Section~\ref{sec:amazon} of the Appendix for results on the \texttt{Amazon-670k} extreme classification dataset.

\subsection{Maximum inner product search retrieval}
Next, we compare our MIPS retrieval performance against other quantization techniques at equal bitrate.
We compare to LSQ~\cite{LSQ} and all three variants of QUIPS~\cite{QUIPS}. In Figure~\ref{fig:fixedbitrate} we measure the performance at fixed bitrates of 100 and 200 bits per datapoint. Our metric is Recall 1@N, which corresponds to the proportion of queries where the top $N$ retrieved results contain the true top-1 datapoint. Our algorithm using score-aware loss outperforms other algorithms at both bitrates and all ranges of $N$. 

Other quantization methods may also benefit from using score-aware quantization loss. For example, binary quantization techniques such as~\cite{stochastic_hashing} use reconstruction loss in their original paper, but can be easily adapted to the proposed loss by a one line change to the loss objective.  We show results which illustrate the improvement of such a change in Section~\ref{sec:binary_quantization} of Appendix.

\subsection{Recall-Speed benchmark} \label{sec:speed_benchmark}
Fixed-bit-rate experiments mostly compare asymptotic behavior and often overlook preprocessing overhead such as learned rotation or lookup table computation, which can be substantial. To evaluate effectiveness of MIPS algorithms in a realistic setting, it is important to perform end-to-end benchmarks and compare speed-recall curves. We adopted the methodology of public benchmark ANN-Benchmarks~\cite{aumuller2019ann}, which plots a comprehensive set of 11 algorithms for comparison, including \texttt{faiss}~\cite{FAISS} and \texttt{hnswlib}~\cite{hnsw}.

Our benchmarks are all conducted on an Intel Xeon W-2135 with a single CPU thread, and followed the benchmark's protocol. Our implementation builds on product quantization with the proposed quantization and SIMD based ADC~\cite{QUIPS} for distance computation. This is further combined with a vector quantization based tree~\cite{MSQ}. Our implementation is open-source and available at \url{https://github.com/google-research/google-research/tree/master/scann} and furthermore the exact configurations used to produce our benchmark numbers are part of the ANN-Benchmarks GitHub repository. Figure~\ref{fig:ann-benchmarks} shows our performance on \texttt{Glove1.2M} significantly outperforms competing methods in the high-recall region.

\section{Conclusion}

In this paper, we propose a new quantization loss function for inner product search, which replaces traditional reconstruction error. The new loss function is weighted based on the inner product values, giving more weight to the pairs of query and database points with higher inner product values. The proposed loss function is theoretically proven and can be applied to a wide range of quantization methods, for example product and binary quantization. Our experiments show superior performance on retrieval recall and inner product value estimation compared to methods that use reconstruction error. The speed-recall benchmark on public datasets further indicates that the proposed method outperforms state-of-the-art baselines which are known to be hard to beat.

\FloatBarrier
\bibliographystyle{icml2020}
\bibliography{reference}

\clearpage
\onecolumn
\section{Appendix} \label{sec:appendix}
\subsection{Proof of Theorem \ref{thm:general_w}}\label{proof:general_w}
We first prove the following lemma:
\begin{lemma}
Suppose we are given a datapoint $x$ and its quantization $\tilde{x}$. If $q$ is uniformly spherically distributed, then

$$\mathbb{E}_q[\inner{q, x-\tilde{x}}^2 | \inner{q,x}=t] = \dfrac{t^2}{||x||^2}||r_\parallel(x,\tilde{x})||^2 + \dfrac{1-\frac{t^2}{||x||^2}}{d-1}||r_\perp(x,\tilde{x})||^2$$

with $r_\parallel$ and $r_\perp$ defined as in section \ref{sec:analyze_general_w}.
\end{lemma}
\begin{proof}
First, we can decompose $q :=q_{\parallel}+q_{\perp}$ with $q_{\parallel} :=\inner{q,x} \cdot \frac{x}{||x||}$ and  $q_{\perp} :=q - q_{\parallel}$ where $q_{\parallel}$ is parallel to $x$ and $q_{\perp}$ is orthogonal to $x$. Then, we have

\begin{align}
\mathbb{E}_{q} [\inner{q,  x - \quant{x}} ^2 | \inner{q, x}=t] &=  \mathbb{E}_q [\inner{  q_{\parallel}+q_{\perp}, r_{\parallel}(x, \quant{x})+r_{\perp}(x, \quant{x})} ^2 | \inner{q, x}=t] \nonumber \\
&=  \mathbb{E}_q [(\inner{q_{\parallel}, r_{\parallel}(x, \quant{x})}  + \inner{q_{\perp},r_{\perp}(x, \quant{x})} )^2 | \inner{q, x}=t] \nonumber \\
&=  \mathbb{E}_q [\inner{q_{\parallel}, r_{\parallel}(x, \quant{x})}^2 | \inner{q, x}=t] + \mathbb{E}_q [\inner{q_{\perp}, r_{\perp}(x, \quant{x})}^2 | \inner{q, x}=t],\label{eq:expectation_derive}
\end{align}

The last step uses the fact that $\mathbb{E}_q [\inner{q_{\parallel}, r_{\parallel}(x, \quant{x})}  \inner{q_{\perp},r_{\perp}(x, \quant{x})}  | \inner{q, x}=t] = 0$ due to symmetry. The first term of \eqref{eq:expectation_derive}, $\mathbb{E}_q [\inner{q_{\parallel}, r_{\parallel}(x, \quant{x})}^2 | \inner{q, x}=t] = \|r_{\parallel}(x, \quant{x})\|^2 \mathbb{E}_q [\|q_{\parallel}\|^2 | \inner{q, x}=t] = \frac{\|r_{\parallel}\|^2 t^2}{\|x\|^2}$. 
For the second term, since $q_{\perp}$ is uniformly distributed in the $(d-1)$ dimensional subspace orthogonal to $x$ with the norm $\sqrt{1-\frac{t^2}{\|x\|^2}}$, we have $\mathbb{E}_q [\inner{q_{\perp}, r_{\perp}(x, \quant{x})}^2 | \inner{q, x}=t] = \frac{1-\frac{t^2}{\|x\|^2}}{d-1} ||r_{\perp}(x, \quant{x})||^2$. Therefore

$$\mathbb{E}_{q} [\inner{q, r(x, \quant{x})} ^2 | \inner{q, x}=t] = \frac{t^2}{\|x\|^2} ||r_{\parallel}(x, \quant{x})||^2+\frac{1-\frac{t^2}{\|x\|^2}}{d-1} ||r_{\perp}(x, \quant{x})||^2.$$
\end{proof}
\begin{proof}[Proof of Theorem \ref{thm:general_w}]
We can expand $\ell(x_i,\tilde{x}_i, w)$ as

$$\int_{-||x_i||}^{||x_i||} w(t)\mathbb{E}_q[\inner{q,x_i-\tilde{x}_i}^2 | \inner{q,x_i}=t] d\text{P}(\inner{q,x_i}\le t)$$

Let $\theta :=\arccos\dfrac{t}{||x_i||}$ so $t=||x_i||\cos\theta$. Because we are assuming $q$ is uniformly spherically distributed, $\frac{d\text{P}(\inner{q,x} \le t)}{d t}$ is proportional to the surface area of $(d-1)$-dimensional hypersphere with a radius of $\sin \theta$. Thus we have $\frac{d\text{P}(\inner{q,x}=t)}{dt} \propto S_{d-1} \sin^{d-2}\theta$, where $S_{d-1}$ is the surface area of $(d-1)$-sphere with unit radius. Our integral can therefore be written as:

$$\int_0^\pi w(||x_i||\cos\theta)\mathbb{E}_q[\inner{q,x_i-\tilde{x}_i}^2 | \inner{q,x_i}=||x_i||\cos\theta] \sin^{d-2}\theta d\theta.$$

Using our above lemma this simplifies to

$$\int_0^\pi w(||x_i||\cos\theta)\left(\cos^2\theta ||r_{\parallel}(x, \quant{x})||^2+\frac{\sin^2\theta}{d-1} ||r_{\perp}(x, \quant{x})||^2\right) \sin^{d-2}\theta d\theta.$$

From here we can clearly see that

$$\begin{aligned}
\ell(x_i,\tilde{x}_i, w) &= h_\parallel(w, ||x_i||)||r_\parallel(x_i,\tilde{x}_i)||^2+h_\perp(w, ||x_i||)||r_\perp(x_i,\tilde{x}_i)||^2,\\
h_\parallel &:=\int_0^\pi w(||x_i||\cos\theta)(\sin^{d-2}\theta-\sin^d\theta)d\theta,\\
h_\perp &:=\dfrac{1}{d-1}\int_0^\pi w(||x_i||\cos\theta)\sin^d\theta d\theta
\end{aligned}$$

as desired.
\end{proof}

\subsection{Proof of Theorem \ref{thm:parallel_greater}}\label{proof:parallel_greater}
\begin{proof}[Proof of Theorem \ref{thm:parallel_greater}]
Note that $h_\parallel$ and $h_\perp$ equal zero if and only if $w(t)=0$ for $t\in[-||x_i||, ||x_i||]$. Otherwise both quantities are strictly positive so it is equivalent to prove that $\dfrac{h_\parallel(w, ||x_i||)}{h_\perp(w, ||x_i||)}\ge1$ with equality if and only if $w$ is constant.

$$\begin{aligned}
\dfrac{h_\parallel(w, ||x_i||)}{h_\perp(w, ||x_i||)} &= \dfrac{\displaystyle\int_0^\pi w(||x_i||\cos\theta)(\sin^{d-2}\theta-\sin^d\theta)d\theta}{\dfrac{1}{d-1}\displaystyle\int_0^\pi w(||x_i||\cos\theta)\sin^d\theta d\theta} \\
&= (d-1)\left(\dfrac{\int_0^\pi w(||x_i||\cos\theta)\sin^{d-2}\theta d\theta}{\int_0^\pi w(||x_i||\cos\theta)\sin^d\theta d\theta}-1\right)
\end{aligned}$$

Define $I_d:=\int_0^\pi w(||x_i||\cos\theta)\sin^d\theta d\theta$. Our objective is to prove $(d-1)\left(\dfrac{I_{d-2}}{I_d}-1\right)\ge1$ or equivalently $\dfrac{I_{d-2}}{I_d}\ge\dfrac{d}{d-1}$. To do this we use integration by parts on $I_d$:

$$\begin{aligned}
I_d=&-w(||x_i||\cos\theta)\cos\theta\sin^{d-1}\theta\Big|_0^\pi+\\
&\int_0^\pi \cos\theta\left[w(||x_i||\cos\theta)(d-1)\sin^{d-2}\theta\cos\theta-w'(||x_i||\cos\theta)||x_i||\sin^d\theta\right] d\theta\\
=&(d-1)\int_0^\pi w(||x_i||\cos\theta)\cos^2\theta\sin^{d-2}\theta-||x_i||\int_0^\pi w'(||x_i||\cos\theta)\cos\theta\sin^d\theta d\theta\\
=&(d-1)I_{d-2}-(d-1)I_d-||x_i||\int_0^\pi w'(||x_i||\cos\theta)\cos\theta\sin^d\theta d\theta \label{eq:parallel_greater_proof_eq1}
\end{aligned}$$

We now show that $\int_0^\pi w'(||x_i||\cos\theta)\cos\theta\sin^d\theta d\theta\ge0$ with equality if and only if $w$ is constant. As a prerequisite for this theorem $w(t)=0$ for $t<0$ so our integral simplifies to $\int_0^{\pi/2}w'(||x_i||\cos\theta)\cos\theta\sin^d\theta d\theta\ge0$. From 0 to $\pi/2$ both sine and cosine are non-negative. Since $w$ is non-decreasing in this range, $w'\ge0$ and therefore our integral is non-negative. The integral equals zero if and only if $w'=0$ over the entire range of $t$ which implies $w$ is constant.

Applying our inequality to equation \ref{eq:parallel_greater_proof_eq1} we get $\dfrac{I_{d-2}}{I_d}\ge\dfrac{d}{d-1}$ as desired.

\end{proof}

\subsection{Proof of Results for $w(t)=\mathbf{I}(t\ge T)$}\label{proof:normalized-dataset-results}
\begin{proof}[Proof of Equation \ref{eq:eta-limit}]
Let $\alpha:=\arccos(T/||x_i||)$. If we do the same analysis as section \ref{proof:parallel_greater} but specialized for $w(t)=\mathbf{I}(t\ge T)$ we find that $I_d=\int_0^\alpha \sin^d\theta d\theta$ and

$$dI_d=(d-1)I_{d-2}-\cos\alpha\sin^{d-1}\alpha.\label{eq:special_w_recursion}$$

From the Cauchy–Schwarz inequality for integrals, we have

\[
\Big(\int_0^\alpha{\sin^{\frac{d+2}{2}}\theta \sin^{\frac{d-2}{2}}\theta d\theta}\Big)^2 \le \int_0^\alpha{\sin^{d+2}}\theta d\theta \int_0^\alpha{\sin^{d-2}\theta d\theta}
\]

Rearranging this we have $\frac{I_d}{I_{d+2}} \le \frac{I_{d-2}}{I_d}$, which proves that $\frac{I_{d-2}}{I_d}$ is monotonically non-increasing as $d$ increases. From section \ref{proof:parallel_greater} we already have a lower bound $\frac{I_{d-2}}{I_d}>1$. Since the ratio is monotonically non-increasing, $\lim_{d\to\infty}\frac{I_d}{I_{d+2}}$ exists. 

Dividing both sides of equation~\ref{eq:special_w_recursion} by $dI_d$, we have
\[
1=\frac{-\cos \alpha \sin^{d-1} \alpha}{d I_d} + \frac{(d-1) I_{d-2}}{d I_d} 
\]

Using our above analysis we know that $\displaystyle\lim_{d\to\infty}\frac{(d-1)I_{d-2}}{dI_d}$ exists so therefore $\displaystyle  \lim_{d\to \infty} \frac{\cos \alpha \sin^{d-1} \alpha}{d I_d} > 0$ also exists. Furthermore,
\[
\lim_{d\to \infty} \frac{\frac{\cos \alpha \sin^{d-1} \alpha}{d I_d}}{\frac{\cos \alpha \sin^{d-3} \alpha}{(d-2) I_{d-2}}}=1 \Rightarrow
\lim_{d\to \infty} \frac{(d-2) I_{d-2}}{dI_d} = \frac{1}{\sin^2 \alpha}
\]
Finally we have $\displaystyle \lim_{d \to \infty} \frac{\eta(\mathbf{I}(t\ge T), ||x_i||)}{d-1}=\frac{1}{\sin^2 \alpha} - 1=\frac{(T/||x_i||)^2}{1-(T/||x_i||)^2}$, and this proves equation~\ref{eq:eta-limit}.
\end{proof}

\subsection{Proof of Theorem \ref{thm:anisotropic-update}}
\label{proof:anisotropic-update}
 \begin{proof}[Proof of Theorem \ref{thm:anisotropic-update}]

Consider a single point $x_i$ with $r_\parallel := r_\parallel(x_i, \quant{x}_i) = \frac{1}{\|x\|^2}x_i x_i^T (x_i - \quant{x}_i)$ and $r_\perp := r_\perp(x_i, \quant{x}_i) = x_i - \quant{x}_i - r_\parallel$. We have that
\begin{align}
\|r_\perp\|^2 &= (x_i - \quant{x}_i - r_\parallel)^T(x_i - \quant{x}_i - r_\parallel) \notag \\
&= \|x_i\|^2 + \|\quant{x}_i\|^2 - 2 x_i^T \quant{x}_i - 2 r_\parallel^T(x_i - \quant{x}_i) + \|r_\parallel\|^2 \notag\\
&= \|x_i\|^2 + \|\quant{x}_i\|^2 - 2 x_i^T \quant{x}_i - \|r_\parallel\|^2 \label{r_perp_norm},
\end{align}
where we use the fact that $x_i - \quant{x}_i = r_\parallel + r_\perp$ and $r_\parallel$ is orthogonal to $r_\perp$.

We also have
\begin{align}
    \|r_\parallel\|^2 &= \frac{1}{\|x_i\|^4} \left ( x_i(x - \quant{x}_i)^T x_i \right)^T \left ( x_i(x - \quant{x}_i)^T x_i \right) \notag \\
    &= \frac{1}{\|x_i\|^4} x_i^T (x_i - \quant{x}_i)x_i^T x_i (x_i - \quant{x}_i)^T x_i \notag\\
    &= \frac{1}{\|x_i\|^2} x_i^T (x_i - \quant{x}_i)(x_i - \quant{x}_i)^T x_i \notag \\
    &= \|x_i\|^2  + \frac{\quant{x}_i^T x_i x_i^T \quant{x}_i}{\|x_i\|^2} - 2 x_i^T \quant{x}_i.\label{r_parallel_norm}
\end{align}

Combining Equations \eqref{r_perp_norm} and \eqref{r_parallel_norm}, we have that
\[
h_{i, \parallel} \|r_\parallel\|^2 + h_{i, \perp} \|r_\perp\|^2 = \quant{x}_i^T\left((h_{i, \parallel} - h_{i, \perp})\frac{x_i x_i^T}{\|x_i\|^2} + h_{i, \perp} I\right)\quant{x}_i - 2 h_{i, \parallel}x_i^T \quant{x}_i + h_{i, \parallel}\|x_i\|^2.
\]

Ignoring the constant term  $h_{i, \parallel}\|x_i\|^2$ and summing over all datapoints $x_i$ that have $\quant{x}$ as a center, we have that the total loss is equivalent to
\begin{equation}\label{single-point-loss}
\quant{x}^T \left (\sum_i  (h_{i, \parallel} - h_{i, \perp})\frac{x_i x_i^T}{\|x_i\|^2} + h_{i, \perp} I\right)\quant{x} - 2 \left(\sum_i h_{i, \parallel}x_i \right)^T \quant{x}.
\end{equation}

Since we established in Theorem \ref{thm:parallel_greater} that $h_{i, \parallel} \geq h_{i, \perp}$, we have that the loss function is a convex quadratic function and thus we can calculate the optimal value of $\quant{x}$ as
\[
\quant{x} = \left (\sum_i  (h_{i, \parallel} - h_{i, \perp})\frac{x_i x_i^T}{\|x_i\|^2} + h_{i, \perp} I\right)^{-1}\left(\sum_i h_{i, \parallel}x_i \right).
\]

\end{proof}

\subsection{Codebook Optimization in Product Quantization}
\label{sec:pq_update}

Let $c$ be a vector with all dictionary codewords. We can get a quantized point $\quant{x}_i$ by calculating $B c$, where $B$ is a $\{0, 1\}$-matrix with dimensions $d \times dk$ that selects the relevant codewords.

For example, suppose $\{(-1, -1), (1, 1)\}$ are our codewords for the first two dimensions and $\{(-2, -2), (2, 2)\}$ are our codewords for the next two dimensions. We have our vectorized dictionary $c = (-1, -1, 1, 1, -2, -2, 2, 2)$. If we want to represent $(-1, -1, 2, 2)$, we set $B$ to be
\[
\begin{pmatrix}
1 & 0 & 0 & 0 & 0 & 0 & 0 & 0 \\
0 & 1 & 0 & 0 & 0 & 0 & 0 & 0 \\
0 & 0 & 0 & 0 & 0 & 0 & 1 & 0 \\
0 & 0 & 0 & 0 & 0 & 0 & 0 & 1
\end{pmatrix}.
\]
Similarly, if we want to represent $(1, 1, -2, -2)$ we set $B$ to be
\[
\begin{pmatrix}
0 & 0 & 1 & 0 & 0 & 0 & 0 & 0 \\
0 & 0 & 0 & 1 & 0 & 0 & 0 & 0 \\
0 & 0 & 0 & 0 & 1 & 0 & 0 & 0 \\
0 & 0 & 0 & 0 & 0 & 1 & 0 & 0
\end{pmatrix}.
\]

We can now write $\quant{x}_i$ as $\quant{x}_i = B_i c$ for some matrix $B_i$.

To minimize our loss function over $c$, we start by summing over Equation \ref{single-point-loss} and ignoring all constant terms to get
\[
c^T \left ( \sum_i B_i^T\left ((h_{i, \parallel} - h_{i, \perp})\frac{x_i x_i^T}{\|x_i\|^2} + h_{i, \perp} I\right) B_i\right )c - 2 \left (  \sum_i h_{i, \parallel}B_ix_i \right)c.
\]

This is again a convex quadratic minimization problem over $c$ and can be solved efficiently. Specifically the matrix
\[
\sum_i B_i^T\left ((h_{i, \parallel} - h_{i, \perp})\frac{x_i x_i^T}{\|x_i\|^2} + h_{i, \perp} I\right)B_i
\]
will be full rank if we observe every codeword at least once. We can then find the optimal value of $c$ with
\[
c = \left ( \sum_i B_i^T\left ((h_{i, \parallel} - h_{i, \perp})\frac{x_i x_i^T}{\|x_i\|^2} + h_{i, \perp} I\right)B_i\right)^{-1}\left (  \sum_i h_{i, \parallel}B_ix_i \right).
\]

\subsection{Results on the \texttt{Amazon-670k} Extreme Classification Dataset}
\label{sec:amazon}
Extreme classification with a large number of classes requires evaluating the last layer (classification layer) with all possible classes. When there are $\mathcal{O}(M)$ classes, this becomes a major computation bottleneck as it involves a huge matrix multiplication followed by Top-K. Thus this is often solved using Maximum Inner Product Search to accelerate inference.
We evaluate our methods on extreme classification using the \texttt{Amazon-670k} dataset~\cite{amazon670}. An MLP classifier is trained over 670,091 classes, where the last layer has a dimensionality of 1,024. The retrieval performance of product quantization with traditional reconstruction loss and with score-aware quantization loss are compared in Table~\ref{fig:amazon-table}.

\begin{table}
\centering
\begin{tabular}{|l|l|l|l||l||l|l|l|l|}
\hline
Bitrate            & 1@1            & 1@10               & 1@100         & Bitrate            & 1@1            & 1@10               & 1@100  \\ \hline
\hhline{|========|}
256 bits, PQ   & 0.652          & 0.995                 & 0.999    & 512 bits, PQ         & 0.737          & 0.998             & 1.000      \\ \hline     
256 bits, Ours & \textbf{0.656} & \textbf{0.996} & \textbf{1.000}  & 512 bits, Ours & \textbf{0.744} & 0.997  & 1.000    \\ \hline
\hhline{|========|}
1024 bits, PQ         & 0.778          & 1.000             & 1.000 & 2048 bits, PQ         & 0.782          & 1.000             & 1.000      \\ \hline     
1024 bits, Ours & \textbf{0.812} & 1.000  & 1.000  & 2048 bits, Ours         & \textbf{0.875}          & 1.000             & 1.000  \\ \hline
\end{tabular}
\caption{\texttt{Amazon-670k} extreme classification performance. The benefits of anisotropic vector quantization on Recall 1@$N$are especially evident at lower bitrates and lower $N$.}\label{fig:amazon-table}
\end{table}

\subsection{Results on Binary Quantization}
\label{sec:binary_quantization}
Another popular family of quantization function is binary quantization. In such a setting, a function $h(x): \R^d \rightarrow \{0, 1\}^{h}$ is learned to quantize datapoints into binary codes, which saves storage space and can speed up distance computation. There are many possible ways to design such a binary quantization function, and some~\cite{carreira2015hashing,stochastic_hashing} uses reconstruction loss. 

We can apply our score-aware quantization loss to these approaches. We follow the setting of Stochastic Generative Hashing (SGH)~\cite{stochastic_hashing}, which explicitly minimizes reconstruction loss and has been shown to outperform earlier baselines. In their paper, a binary auto-encoder is learned to quantize and dequantize binary codes:
\begin{equation*}
\tilde x = g(h(x)); \text{where~} h(x) \in \{0, 1\}^{h}
\end{equation*}
where $h(\cdot)$ is the ``encoder'' part which binarizes original datapoint into binary space and $g(\cdot)$ is the ``decoder'' part which reconstructs the datapoints given the binary codes. The authors of the paper uses $h(x) = sign(W_h^T x+ b_h)$ as the encoder function and $g(h) = W_g^T h$ as the decoder functions. The learning objective is to minimize the reconstruction error of $||x-\tilde x||^2$, and the weights in the encoder and decoder are optimized end-to-end using standard stochastic gradient descent. We can instead use our score-aware quantization loss. We show below the results of SGH and SGH with our score-aware quantization loss in Table \ref{sgh-score-aware} on the \texttt{SIFT1M} dataset \citep{PQ}. We see that adding our score-aware quantization loss greatly improves performance.

\begin{table}[h]
\centering
\begin{tabular}{|l|l|l|l|l|}
\hline
Recall $k@k$    & 1@1            & 1@10           & 10@10     & 10@100  \\ \hline
64 bits, SGH         & 0.028          & 0.096          & 0.053          & 0.220     \\ \hline     
64 bits, SGH-score-aware & \textbf{0.071} & \textbf{0.185} & \textbf{0.093} & \textbf{0.327} \\ \hline
\hhline{|=====|}
128 bits, SGH         & 0.073          & 0.195          & 0.105          & 0.376     \\ \hline     
128 bits, SGH-score-aware & \textbf{0.196} & \textbf{0.406} & \textbf{0.209} & \textbf{0.574} \\ \hline
\hhline{|=====|}
256 bits, SGH         & 0.142          & 0.331          & 0.172          & 0.539     \\ \hline     
256 bits, SGH-score-aware & \textbf{0.362} & \textbf{0.662} & \textbf{0.363} & \textbf{0.820} \\ \hline
\end{tabular}
\caption{We compare Stochastic Generative Hashing \citep{stochastic_hashing} trained with reconstruction loss (SGH) and Stochastic Generative Hashing trained with our score-aware quantization loss (SGH-score-aware) on the \texttt{SIFT1M} dataset. We see that using our score-aware loss greatly improves the recall of Stochastic Generative Hashing.}\label{sgh-score-aware}
\end{table}

\subsection{Dataset Selection for MIPS evaluation}
\label{sec:idiosyncrasy}
In this section we consider dataset choices for benchmarking MIPS systems. In modern large-scale settings, the vectors in the database are often created with neural network embeddings learned by minimizing some training task. This typically leads to the following nice properties:
\begin{itemize}
    \item Low correlation across dimensions.
    \item Equal variance in each dimension.
\end{itemize}
Since our target application is retrieval in such settings, we want our benchmarking dataset to have these properties. This will allow our metrics to better inform how our approach will work in practice.

Datasets that have been widely used for evaluating MIPS systems include \texttt{SIFT1M/1B}, \texttt{GIST1M}, \texttt{Glove1.2M}, \texttt{Movielens}, and \texttt{Netflix}. We see in Figure \ref{dataset-plots} that only \texttt{Glove1.2M} has the properties we want in a benchmarking dataset.

\textbf{\texttt{SIFT1M}, \texttt{SIFT1B}, and \texttt{GIST1M}} are introduced by~\cite{PQ} to illustrate the use of product quantization. \texttt{SIFT} is a keypoint descriptor while GIST is image-level descriptor which have been hand-crafted for image retrieval. These vectors have a high correlation between dimensions and have a high degree of redundancy. Thus the intrinsic dimensions of \texttt{SIFT1M} and \texttt{GIST} are much lower than its dimensionality. 

\textbf{\texttt{Movielens and Netflix}} dataset are formed from the SVD of the rating matrix of Movielens and Netflix websites, respectively. This is introduced by~\cite{ALSH} for MIPS retrieval evaluation. Following SVD of $X = (U \Lambda^{1/2T})(\Lambda^{1/2}V)$, the dimension of these two datasets correspond to the eigenvalues of $X$. Thus the variance of dimensions are sorted by eigenvalues, and the first few dimensions are much more important than later ones. Additionally, the datasets are 10k - 20k in size and thus should not be considered large-scale.

\textbf{\texttt{Glove1.2M}} is a word embeddings dataset similar to word2vec, which use neural-network style training with a bottleneck layer. This datasets exhibits less data distribution problems. It is our general observation that bottleneck layers lead to independent dimensions with similar entropy, making them good datasets for benchmarking for our target retrieval tasks.

\begin{figure}
\centering
\begin{tabular}{|l|l|l|l|}
\hline
Dataset    & Size            & Correlation           & Variance by dimension  \\ \hline
\texttt{SIFT1M}    & (1000000, 128) & \includegraphics[width=0.25 \textwidth]{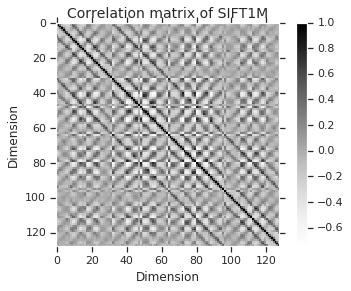} & \includegraphics[width=0.3 \textwidth]{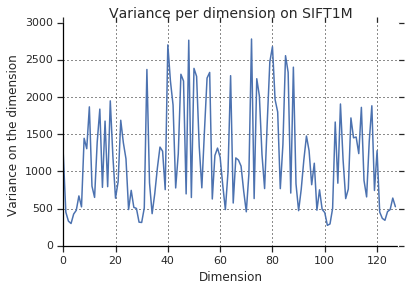}  \\ \hline
\texttt{GIST1M}    & (1000000, 960) & \includegraphics[width=0.25 \textwidth]{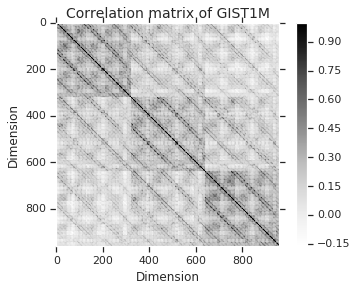} & \includegraphics[width=0.3 \textwidth]{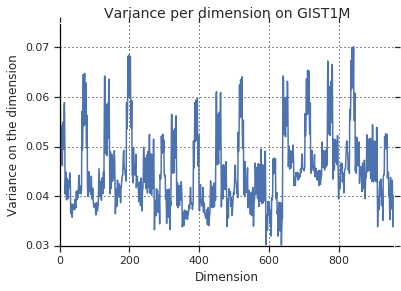}  \\ \hline
\texttt{MovielensSVD}    & (10681, 150) & \includegraphics[width=0.25 \textwidth]{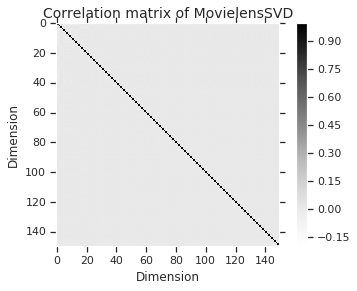} & \includegraphics[width=0.3 \textwidth]{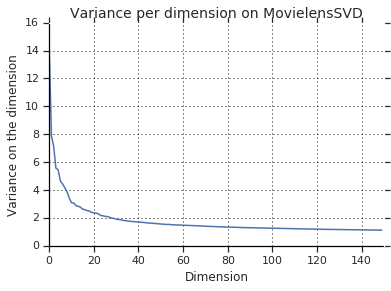}  \\ \hline
\texttt{NetflixSVD}    & (17770, 300) & \includegraphics[width=0.25 \textwidth]{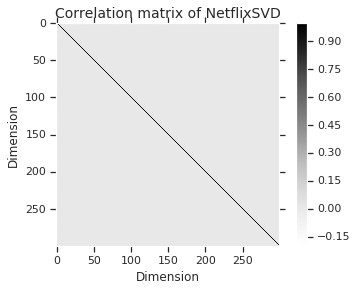} & \includegraphics[width=0.3 \textwidth]{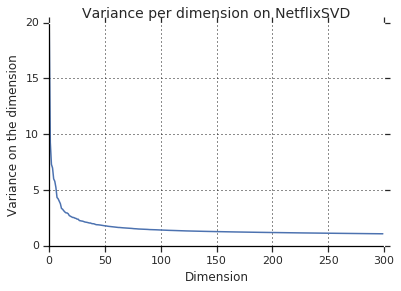}  \\ \hline
\texttt{Glove1.2M}    & (1183514, 100) & \includegraphics[width=0.25 \textwidth]{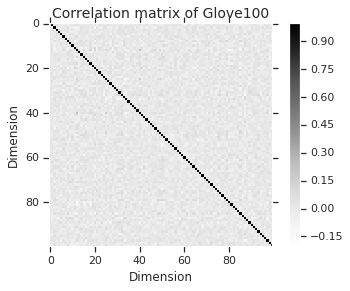} & \includegraphics[width=0.3 \textwidth]{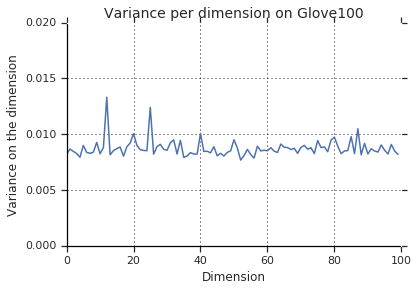}  \\ \hline
\end{tabular}
\caption{We plot the correlation and variance by dimensions of \texttt{SIFT1M}, \texttt{GIST1M}, \texttt{MovielensSVD}, \texttt{NetflixSVD}, and \texttt{Glove1.2M}. We see that \texttt{SIFT1M} and \texttt{GIST1M} have strong correlations between dimensions, and thus their intrinsic dimensions are significantly lower than the original dimensions. We see that \texttt{MovielensSVD} and \texttt{NetflixSVD} suffers from problem of a large variation in the variance across dimensions. In contrast, \texttt{Glove1.2M} has nearly uncorrelated dimensions and roughly equal variance across dimensions, making it a good dataset for our target retrieval tasks.}\label{dataset-plots}
\end{figure}

\end{document}